\documentclass{article}

\usepackage[a4paper, total={6in, 8.5in}]{geometry}
\usepackage[utf8]{inputenc} 
\usepackage[T1]{fontenc}    
\usepackage{url}            
\usepackage{booktabs}       
\usepackage{amsfonts}       
\usepackage{nicefrac}       
\usepackage{microtype}      

\usepackage{times}
\usepackage{graphicx}
\usepackage{algorithm}
\usepackage{algorithmic}
\usepackage{amsmath}
\usepackage{amssymb}
\usepackage{amsthm}
\usepackage{epsfig}
\usepackage{epstopdf}
\usepackage{mathtools}
\usepackage{multirow}
\usepackage{placeins}
\usepackage{expl3}
\usepackage{placeins}
\usepackage{bbm}
\usepackage{enumitem}
\usepackage{titlesec}
\usepackage{caption}
\usepackage{color}
\usepackage{subcaption}

\ExplSyntaxOn
\newcommand\latinabbrev[1]{
  \peek_meaning:NTF . {
    #1\@}%
  { \peek_catcode:NTF a {
      #1.\@ }%
    {#1.\@}}}
\ExplSyntaxOff

\def\eg{\latinabbrev{e.g}}
\def\etal{\latinabbrev{et al}}
\def\etc{\latinabbrev{etc}}
\def\ie{\latinabbrev{i.e}}
\DeclareMathOperator*{\argmin}{arg\,min}

\title{Neural Networks with Smooth Adaptive Activation Functions for Regression}

\author{
Le Hou$^1$, Dimitris Samaras$^1$, Tahsin M. Kurc$^{2,3}$,
       Yi Gao$^{2,1,4}$, Joel H. Saltz$^{2,1,5,6}$
       \\
       $^1$Department of Computer Science, Stony Brook University, NY, USA\\
       \{lehhou,samaras\}@cs.stonybrook.edu\\
       $^2$Department of Biomedical Informatics, Stony Brook University, NY, USA\\
       \{tahsin.kurc,joel.saltz\}@stonybrook.edu\\
       $^3$Oak Ridge National Laboratory, USA\\
       $^4$Department of Applied Mathematics and Statistics, NY, USA\\
       $^5$Department of Pathology, Stony Brook Hospital, NY, USA\\
       yi.gao@stonybrookmedicine.edu\\
       $^6$Cancer Center, Stony Brook Hospital, NY, USA\\
}

\begin{document}

\maketitle

\begin{abstract}
In Neural Networks (NN), Adaptive Activation Functions (AAF) have parameters that control the shapes of activation functions. These parameters are trained along with other parameters in the NN. AAFs have improved performance of Neural Networks (NN) in multiple classification tasks. In this paper, we propose and apply AAFs on feedforward NNs for regression tasks. We argue that applying AAFs in the regression (second-to-last) layer of a NN can significantly decrease the bias of the regression NN. However, using existing AAFs may lead to overfitting. To address this problem, we propose a Smooth Adaptive Activation Function (SAAF) with piecewise polynomial form which can approximate any continuous function to arbitrary degree of error. NNs with SAAFs can avoid overfitting by simply regularizing the parameters. In particular, an NN with SAAFs is Lipschitz continuous given a bounded magnitude of the NN parameters. We prove an upper-bound for model complexity in terms of fat-shattering dimension for any Lipschitz continuous regression model. Thus, regularizing the parameters in NNs with SAAFs avoids overfitting. We empirically evaluated NNs with SAAFs and achieved state-of-the-art results on multiple regression datasets.
\end{abstract}

\section{Introduction}
Neural Networks (NNs), especially Convolutional Neural Networks (CNNs), improved the state-of-the-art on multiple classification tasks~\cite{he2015delving} and regression tasks~\cite{belagiannis2015robust,szegedy2013deep}. We advocate the use of Adaptive Activation Functions (AAF) in NNs applied to regression problems for two reasons. First, recent studies showed that AAFs improve the classification performance of NNs~\cite{agostinelli2014learning,he2015delving,jin2015deep}. Second, the output of a regression NN should be accurate for a range of ground truth values, as supposed to only two binary labels. Thus, a NN with a fixed architecture tends to have larger biases in regression tasks, compared to classification tasks. To address this problem, we argue that applying AAFs on the regression (second-to-last) layer can reduce the model bias in regression problems more efficiently than adding more neurons.

In contrast to conventional non-adaptive activation functions, AAFs have parameters that are trained along with other parameters in the NN. It is rather challenging to construct and apply AAFs. If an AAF is too simple, it may not be able to approximate the optimal activation function to a desired degree of approximation error, especially for regression problems. On the other hand, complex AAFs might lead to severe overfitting. Designing the AAF with the right approximation power and complexity for each application is contingent on experience and trial-and-error.

\begin{figure}[h!]
\begin{center}
   \includegraphics[width=0.995\linewidth,trim={0 0 0 0},clip]{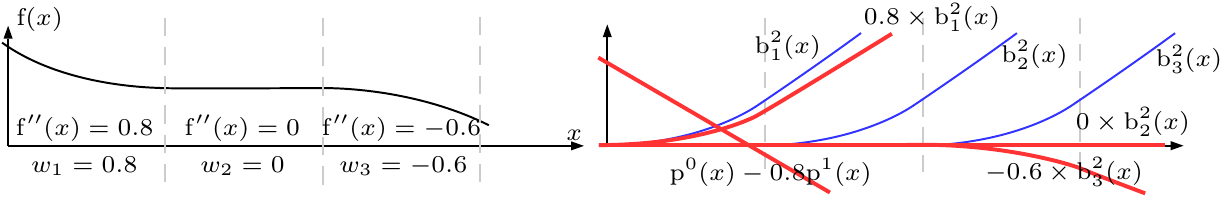}
\end{center}
   \caption{An illustration of the construction of the proposed SAAF with piecewise polynomial form (best viewed in color). Left: a piecewise quadratic SAAF $\mathrm{f}(x)$. In each quadratic segment, the second order derivative $\mathrm{f}''(x)$ is defined by weight $w_i$. Regularizing the parameters $w_{1,2,3}$ leads to a small second order derivative. A regression NN with SAAFs is Lipschitz continuous which results in bounded model complexity. Right: $\mathrm{f}(x)$ equals to the summation of the red curves. Eq.~\ref{equ:activation} gives the formal definitions of the red curves and basis functions $p^{1,2}(x)$, $b_{1,2,3}^2(x)$.}
\label{fig:firstfig}
\end{figure}

AAFs with a simpler form utilize predefined functions. Conservative methods adjust the parameters like the slope of a sigmoid-like activation function~\cite{Bai2009}. These AAFs are limited in form but often yield a faster convergence rate. Other predefined functions or combinations of functions with adjustable parameters~\cite{conf/ijcnn/XuZ00,journals/eaai/IsmailJZZ13} improve the approximation ability of AAFs. Using these AAFs led to better classification accuracy or architectures with fewer model parameters. In these approaches, the AAFs are linear combinations of sub-functions or nested sub-functions including the sigmoid, exponential, sine, \etc. Their major drawback is that the set of sub-functions need to be selected carefully for different datasets, so that the optimal shape of the activation function is approximated to a desired degree of error, without introducing too many parameters.

Piecewise polynomials are well-developed tools for constructing general and complex form AAFs. Specific parameterizations, \eg  \;Splines, can handle control points implicitly~\cite{Guarnieri_spline,conf/ijcnn/HongGC11}. However, the training processes of these methods are very complex. Additionally, too many parameters are introduced for each polynomial segment, increasing the probability of severe overfitting. Non-Spline piecewise linear or quadratic parameterizations~\cite{journals/tnn/Hikawa03a,HopfieldLinear} have issues of discontinuity, non-differentiability or unbounded smoothness, due to their parameterizations.

Note that in most of the methods mentioned so far, one global AAF is applied on all neurons. Recent research in deep CNNs proposed to learn an AAF for each layer of neurons or even individual neurons~\cite{goodfellow2013maxout,he2015delving,agostinelli2014learning,jin2015deep}, as an alternative for reducing model bias. Extending the non-adaptive Rectified Linear Units (ReLU) to Parameterized Rectified Linear Units (PReLU)~\cite{he2015delving} introduces a parameter which controls the slope of the activation function. Activation functions at different layers have the same form, but different slope. A maxout neuron~\cite{goodfellow2013maxout} outputs the maximum of a set of linear functions. It can approximate any convex function and achieved state-of-the-art performances on multiple classification tasks. Adaptive Piecewise Linear Units (APLU)~\cite{agostinelli2014learning} learn the position of break points and the slope of linear segments simultaneously during training. As each neuron learns its own maxout function or APLU, the number of parameters in the NNs significantly increases with no clear principles of how to avoid severe overfitting.

\begin{figure*}[h]
\begin{center}
   \includegraphics[width=0.9\linewidth,trim={20 157 22 10},clip]{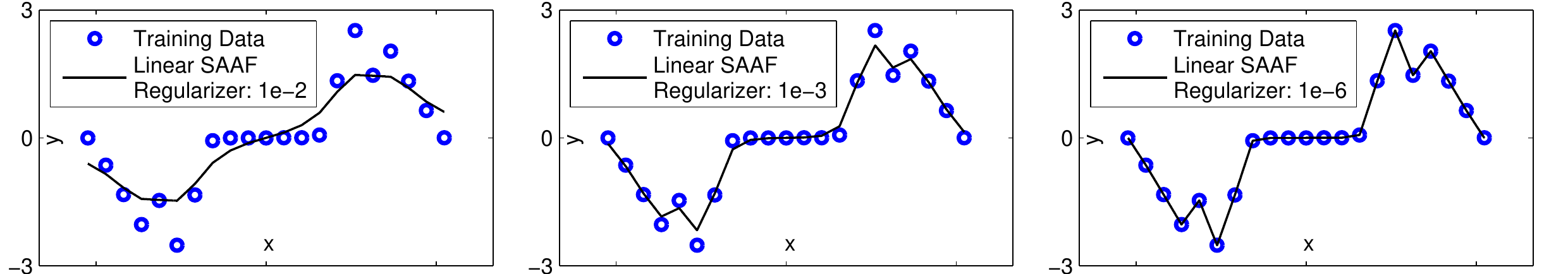}
   \includegraphics[width=0.9\linewidth,trim={20 0 22 157},clip]{figures/smthact-eps-converted-to.pdf}
\end{center}
   \caption{We demonstrate that SAAFs with piecewise polynomial form have regularized smoothness. NNs with SAAFs have bounded model complexity, as explained in Sec.~\ref{sec:properties}. In each plot, there are only 21 training data (input $x$ with ground truth $t$) but 5000 linear or quadratic segments. Thus, most polynomial segments do not have any data to train on. However, the resulting curve is smooth and not overfitting given a reasonable L2 regularization on the parameters in SAAFs. The parameter regularization affects the magnitude of the first and second order derivatives for the piecewise linear and quadratic SAAFs respectively.}
\label{fig:notoverfitting}
\end{figure*}

To conclude, there are two types of existing AAFs: simple AAFs that do not guarantee a bounded approximation ability and complex AAFs that cannot avoid severe overfitting in a principled manner. Viewed in the bias-variance tradeoff paradigm~\cite{bishop2006pattern}, existing AAFs do not guarantee bounded bias and complexity (variance). We propose a novel AAF named Smooth Adaptive Activation Function (SAAF) with piecewise polynomial form. Given a fixed degree of bias or complexity, an SAAF can achieve lower complexity or bias, than existing AAFs. In particular, an SAAF can be regularized under any complexity in terms of the Lipschitz constant of the function and can approximate any function simpler than the given complexity (\ie, with a smaller Lipschitz constant) to an arbitrarily small bias. To regularize SAAF's Lipschitz constant, one can simply apply L2 regularization on its parameters. In contrast, there are no methods that regularize the complexity of existing AAFs in a principled manner. Furthermore, most existing AAFs can not guarantee a bounded function approximation error. We show an upper bound of the fat-shattering dimension for an NN by the Lipschitz constant of its AAFs. Therefore, the Lipschitz constant of AAFs is a good measurement of model complexity. Figs.~\ref{fig:firstfig} and~\ref{fig:notoverfitting} show examples and properties of SAAFs. Our contributions are:
\begin{enumerate}[itemsep=0.0ex,leftmargin=0.5cm]
\item We propose a Smooth Adaptive Activation Function (SAAF) with piecewise polynomial form that can achieve low model bias and complexity \textit{at the same time}:
	\begin{enumerate}
    \item Low model bias: SAAFs can approximate any one-dimensional function to any desired degree of error, given sufficient number of polynomial segments.
	\item Low model complexity: NNs with SAAFs have bounded model complexity in terms of fat-shattering dimension, given a bounded (L2 regularized) magnitude of the parameters, regardless of the number of polynomial segments.
	\end{enumerate}
\item For any Lipschitz continuous regression model, we prove an exact upper-bound for the fat-shattering dimension without other model assumptions. To our best knowledge, such bound does not exist in the literature.
\item We propose to use SAAFs on the regression (second-to-last) layer of regression NNs. Our experimental results are better than current state-of-the-art on multiple pose estimation and age estimation datasets.
\end{enumerate}

\section{Regression neural networks}
\label{sec:desired}
In this section we decompose the learning process of a regression NN to a summation of many one-dimensional function learning processes. Based on this, we argue that applying AAFs on the regression (second-to-last) layer can achieve a small model bias using a small number of parameters. Without loss of generality, assume a regression NN has only one output neuron which outputs a real value $y$ as the prediction, expressed as: $y = \sum\limits_{i=1}^{m}h_i o_i + b$, where $o_{1,2,\dots,m}$ are the outputs from neurons in the previous layer, $h_{1,2,\dots,m}$ are the weights of the output neuron's input synapses, and $b$ is a bias term. We train the NN to minimize the expected regression loss on the training set $E=\argmin\limits_{\theta} \mathrm{E}[(t - y)^2]$, where $\theta \supseteq \{h_{1,2,\dots,m},b\}$ is the set of trainable parameters of the NN. For a multi-layer NN, the output of a neuron $o_i$ is computed by applying the $i$-th activation function $\mathrm{f}_i$ on the input $p_i$, expressed as $o_i = \mathrm{f}_i(p_i)$. Note that $\mathrm{f}_{1,2,\dots,m}$ do not necessarily have the same form. Notably, for almost all of the activation functions, there exists a function $\mathrm{g}_i$ such that $h_i \mathrm{f}_i(p_i) = \mathrm{f}_i\big(\mathrm{g}_i(h_i, p_i) \big)$ holds for all $h_i$, $o_i$. For example, $\mathrm{g}_i(h_i, p_i)=h_i p_i$ for ReLU, Leaky ReLU (LReLU)~\cite{maas2013rectifier} and PReLU. Thus: $y =\sum\limits_{i=1}^{m}\mathrm{f}_i\big(\mathrm{g}_i(h_i, p_i)\big) + b$. In other words, the final prediction of a regression NN is equivalent to the summation of activation function outputs. We denote $\mathrm{g}_i(h_i, p_i)$ as $x_i$, therefore:
\begin{equation}
\begin{split}
y=\sum\limits_{i=1}^{m} \mathrm{f}_i(x_i) + b\text{.}
\end{split}
\label{equ:cuted}
\end{equation}
In the future, we refer to the neurons/layer connected to the output neuron as ``\textbf{regression neurons/layer}''. These regression neurons have activation functions $\mathrm{f}_{1,2,\dots,m}$.

\newtheorem{theorem}{Theorem}
\begin{theorem}
We consider the regression loss $E$ and the set of hypotheses functions defined by Eq.~\ref{equ:cuted}. We assume on a training set, the loss is minimized such that $y = t$. We also assume $x_{1,2,\dots,m}$ are mutually independent on the training set. Then $\mathrm{f}_i(x_i)=\mathrm{E}[t|x_i] + B_i$ for all $i$, where $\mathrm{E}[t|x_i]$ is the expectation of ground truth $t$ given $x_i$ and $B_i$ is a constant.
\label{the:0}
\end{theorem}
\begin{proof}
From the assumption of $y=t$, we have $\sum_i \mathrm{f}_i(x_i) + b=t$ on all training data. Taking the conditional expectation of $x_1$ on both sides, we have $\sum_i \mathrm{E}[\mathrm{f}_i(x_i)|x_1] + b = \mathrm{E}[t|x_1]$. Based on the assumption that $x_{1,2,\dots,m}$ are mutually independent, $\mathrm{E}[\mathrm{f}_1(x_1)|x_1]=\mathrm{f}_1(x_1)$ and $\mathrm{E}[\mathrm{f}_i(x_i)|x_1]=\mathrm{E}[\mathrm{f}_i(x_i)]$ for every $i \geq 2$. Thus $\mathrm{f}_1(x_1)=\mathrm{E}[t|x_1] + B_i$ where $B_i=-(b+\mathrm{E}[\mathrm{f}_2(x_2)]+\mathrm{E}[\mathrm{f}_3(x_3)]+\dots+\mathrm{E}[\mathrm{f}_m(x_m)])$.
\end{proof}

Theorem~\ref{the:0} shows that $\mathrm{f}_i(x_i)$ is a one-dimensional function that approximates $t$ given feature $x_i$, ignoring constant $B$. Therefore, it is very important to be able to learn this one-dimensional function that can achieve small model bias. Although the assumption of mutually independent $x_{1,2,\dots,m}$ in theorem~\ref{the:0} might not always hold, we found in practice that usually $\mathrm{f}_i(x_i)$ does approximates $t$ in real world datasets, as shown in Fig.~\ref{fig:smthact_pose}. If all of the activation functions $\mathrm{f}_{1,2,\dots,m}$ have the same simple form, \eg, ReLU, then all of the features $x_i$ must be correlated with $t$ linearly. To generate those features $x_i$, more neurons and layers are needed. We propose to model $\mathrm{f}_{1,2,\dots,m}$ as AAFs. In this way, we can add a small number of parameters to achieve small model bias efficiently. In our experiments, applying AAF on the regression layer adds less than 1\% to the total number of NN parameters and less than 10\% to the training time.

\section{Smooth adaptive activation function}
\label{sec:smooth}
We introduce the Smooth Adaptive Activation Function (SAAF) formally in this section and show its advantages in Sec.~\ref{sec:properties}. Given $n+1$ real numbers $a_{1,2,\dots,n+1}$ in ascending order and a non-negative integer $c$, using $\mathbbm{1}(\cdot)$ as the indicator function, we define the SAAF as:
\begin{equation}
\begin{array}{l}
\text{\hspace{3.7cm}} \mathrm{f}(x)=\sum\limits_{j=0}^{c-1} v_j \mathrm{p}^{j} (x) + \sum\limits_{k=1}^{n} w_k \mathrm{b}_k^c(x) \text{,} \\

\text{where \quad}\mathrm{p}^j (x)=\cfrac{x^j}{j\,!} \text{\,,\quad} \mathrm{b}_k^c(x)=\underbrace{\iint\ldots\int_0^x}_{c \text{ times}} \mathrm{b}_k^0(\alpha) \> \mathrm{d}^c \alpha \text{\,,\quad} \mathrm{b}_k^0(x)=\mathbbm{1}(a_k \le x < a_{k+1}) \text{.}
\end{array}
\label{equ:activation}
\end{equation}
The SAAF $\mathrm{f}(x)$ is piecewise polynomial. Predefined parameters $c$ and $a_k$ are the degree of polynomial segments and break points respectively. The parameters $w_k$ and $v_j$ are learned during the training stage. $\mathrm{b}_k^c$ and $\mathrm{p}^j$ are basis functions and $\mathrm{f}(x)$ is a linear combination of these basis functions. $\mathrm{b}_k^0$ is the boxcar function. $\mathrm{b}_k^1$ is the integral of $\mathrm{b}_k^0$, which looks like the step function. $\mathrm{b}_k^2$ is the integral of $\mathrm{b}_k^1$, which looks like the ramp function or ReLU. The degree of the polynomial segments in $\mathrm{f}(x)$ is determined by the degree of the basis functions $\mathrm{b}_k^c$ and $\mathrm{p}^j$. Fig.~\ref{fig:firstfig} visualizes the construction of $\mathrm{f}(x)$.

Based on this parameterization, we can see that the order of polynomial segments can be defined to an arbitrary number. This allows the SAAF to have a larger variety of forms compared to existing AAFs. Additionally, for each polynomial segment, there is only one parameter controlling the $c$-th order derivative within the segment. Derivatives of lower order are guaranteed to be continuous across the entire SAAF, including the locations of break points. By regularizing the parameters $w_{1,2,\dots,n}$, the magnitude of the $c$-th order derivative is regularized. In other words, the resulting activation function is smooth. As a result, NNs with SAAFs are smooth functions. Fig.~\ref{fig:notoverfitting} gives examples of one-dimensional function learning.

\section{SAAF properties}
\label{sec:properties}
We discuss properties and advantages of SAAFs in detail. In summary, an SAAF can approximate any one-dimensional function to any desired degree of accuracy given a sufficient number of segments. NNs with SAAFs have bounded fat-shattering dimension when the NN parameters are regularized.

\subsection{SAAFs as universal approximators}
\label{sec:universal}
Piecewise polynomials can approximate any one-dimensional continuous function to any degree of error given sufficiently small polynomial segments~\cite{larson2013finite}. Because the range of a neuron's input is bounded in practice, an SAAF with a sufficient number of polynomial segments can approximate any function to any degree of error. Moreover, an SAAF with a finite Lipschitz constant can approximate any function that has a smaller Lipschitz constant.

\subsection{An NN with SAAFs is Lipschitz continuous}
We show that because the smoothness of an SAAF can be regularized, a feedforward NN with SAAFs is Lipschitz continuous, given a bounded magnitude of the parameters in the NN. A function $\mathrm{f}$ is Lipschitz continuous if there exists a real constant $L$ such that for any $\alpha_1$ and $\alpha_2$ in the domain of $\mathrm{f}$, $\lvert \mathrm{f}(\alpha_1) - \mathrm{f}(\alpha_2) \rvert \le L \lvert\lvert \alpha_1 - \alpha_2\rvert\rvert$. We use the Euclidean norm in this paper. The constant $L$ is the Lipschitz constant of $\mathrm{f}$.

Assuming a bounded range of input $x$, obviously the maximum derivative magnitude of $\mathrm{f}(x)$ is its Lipschitz constant. Therefore, $L$ can be derived by integrating the parameters $w_{1,2,\dots,n}$ and $v_{1,2,\dots,c-1}$.
\begin{equation}
\begin{split}
L & = \max\limits_{x} \,\biggl\lvert\,\underbrace{{\iint\ldots\int_0^x}}_{c - 1 \text{ times}} \mathrm{w}(\alpha) \,\mathrm{d}^{c-1}\alpha + \sum\limits_{j=1}^{c-1}\underbrace{{\iint\ldots\int_0^x}}_{j - 1 \text{ times}} v_j \,\mathrm{d}^{j-1}\alpha\,\biggr\lvert \text{,}
\end{split}
\label{equ:1OrderBound}
\end{equation}
where $\mathrm{w}(\alpha)=\sum_{k=1}^n w_k \mathrm{b}_k^c(\alpha)$. For example, given $c=1$, $L = \max_k \,\lvert w_k \rvert$. Given $c=2$, $L = \max_x\lvert v_1+\int_0^x \mathrm{w}(\alpha)\rvert$. It has been shown~\cite{anthony2009neural} that if the activation functions in an NN are Lipschitz continuous, then the NN is Lipschitz continuous.

Note that NNs with other activation functions such as the Sigmoid, ReLU and PReLU are also Lipschitz continuous given a bounded magnitude of their parameters. Therefore, NNs with combinations of such activation functions and SAAFs are also Lipschitz continuous. However, NNs with these activation functions tend to have large model bias, as argued in Sec~\ref{sec:desired}.

\subsection{Bounded model complexity}
\label{sec:fat-shattering}
In this section, we prove an upper bound of the model complexity in terms of fat-shattering dimension~\cite{bartlett1994fat} for any Lipschitz continous model (function) such as NNs with SAAFs. Given two models with the same training error, the model with a lower fat-shattering dimension has a better expected generalization error~\cite{anthony2009neural}. Upper bounds of the fat-shattering dimension have been proven~\cite{journals/tit/Bartlett98,anthony2009neural} for Lipschitz continuous NNs under assumptions such as the magnitude of NN parameters. We prove an exact upper bound for any Lipschitz continuous regression model with no other model assumptions.

The fat-shattering dimension is a scalar-related dimension defined as follows: Suppose that $F$ is a set of functions mapping from a domain $X$ to $\mathbb{R}$, $D = \left\{{\mathbf{x}_1, \mathbf{x}_2, \dots, \mathbf{x}_z}\right\}$ is a subset of the domain $X$, and $\gamma$ is a positive real number. Then $D$ is $\gamma$-shattered by $F$ if there exist real numbers $t_1, t_2, \dots, t_z$, such that for all $b\in \left\{{0, 1}\right\}^z$, there exists a function $\mathrm{f}_b$ in $F$, such that for all $1\le i \le m$,
\begin{equation}
\begin{array}{l}
\mathrm{f}_b(\mathbf{x}_i) \ge t_i + \gamma \qquad \text{if} \quad b_i = 1,\\
\mathrm{f}_b(\mathbf{x}_i) \le t_i - \gamma \qquad \text{if} \quad b_i = 0\text{.}
\end{array}
\label{equ:fat}
\end{equation}
The fat-shattering dimension of $F$ at scale $\gamma$, denoted as $\operatorname{fat}_F(\gamma)$, is the size of the largest subset $D$ that is $\gamma$-shattered by $F$.

\begin{theorem}
Suppose all data points lay in a d-dimensional cube of unit volume. All functions in function set $F$ have a Lipschitz constant $L$. Then, the set $F$ has bounded fat-shattering dimension: 
\begin{equation}
\operatorname{fat}_F(\gamma) \le d + \frac{L^d d\,!}{\gamma^d \sqrt{2^d(d+1)}}\text{.}
\label{equ:fatbyme}
\end{equation}
\label{the:1}
\end{theorem}
\begin{proof}
For simplicity, we consider shattering two points $\mathbf{x}_1$ and $\mathbf{x}_2$ with $t_1$ and $t_2$ only. Without loss of generality, assume $t_1 \ge t_2$. If $\mathrm{f}_b\in F$ satisfies Eq.~\ref{equ:fat} when $b_1 = 1$, $b_2=0$, then we have $\lvert \mathrm{f}_b(\mathbf{x}_1)-\mathrm{f}_b(\mathbf{x}_2)\rvert \ge \lvert t_1 - t_2 + 2\gamma\rvert \ge 2\gamma$. Denote $s=\lvert\lvert\mathbf{x}_1 - \mathbf{x}_2\rvert\rvert$. According to the definition of Lipschitz continuity, we have $2 \gamma \le sL$. In other words the distance $s$ between two points must be no smaller than $2\gamma/L$ in order for $F$ to possibly $\gamma$-shatter them. Next we derive $\operatorname{fat}_F(\gamma)$ which is the maximum number of points $F$ can $\gamma$-shatter in a cube of unit volume. Those points form a simplex mesh of at least $\operatorname{fat}_F(\gamma)-d$ number of simplexes. The nodes of the simplex mesh are the data points $\mathbf{x}_{1,2,\dots}$ that we expect $F$ to possibly $\gamma$-shatter. The side length of each simplex should be at least $2\gamma/L$. Therefore the total volume of the mesh is no smaller than $V_d (2 \gamma / L)^d (\operatorname{fat}_F(\gamma)-d)$, where $V_d = \sqrt{d+1}\,/(d\,!\sqrt{2^d})$ is the volume of a d-dimensional regular simplex of unit side length. Because we assume the volume of the simplex mesh (where all data points lay) is no greater than $1$, we have $\operatorname{fat}_F(\gamma) \le d + (2 \gamma / L)^{-d}/V_d$. If we expand $V_d$, we derive Eq.~\ref{equ:fatbyme}.
\end{proof}

From theorem~\ref{the:1}, when $L$ decreases, the fat-shattering dimension decreases polynomially. For an NN with SAAFs, since $L$ is bounded by the magnitude of the NN parameters, regularizing the parameters will reduce model complexity polynomially.

\section{Experiments}
\label{sec:experiments}
We tested NNs with piecewise linear (c=1) and quadratic (c=2) SAAFs on six real world datasets. We implemented NNs using Theano~\cite{Bastien-Theano-2012}. In all experiments, we added $10^{-5}$ times the L2 norm of all NN parameters as a regularization term in the loss function. We also use batch normalization~\cite{ioffe2015batch} to speed up the training process. The activation functions we tested are:
\begin{enumerate}[itemsep=0.0ex,leftmargin=0.5cm]
\item Rectified Linear Units (ReLU)~\cite{maas2013rectifier}. Rectifier networks have been successful in many applications.
\item Leaky Rectified Linear Units (LReLU)~\cite{maas2013rectifier}. LReLU has a fixed negative slope $\alpha$ when the input is negative. We set $\alpha=-1/3$ in all experiments.
\item Parametric Rectified Linear Units (PReLU)~\cite{he2015delving}. Compared to LReLU, PReLU is an AAF with a learnable slope. In CNNs, neurons that share the same filter weights also share the same slope.
\item Adaptive Piecewise Linear Units (APLU)~\cite{agostinelli2014learning}. APLU is a piecewise linear parameterization that is different from SAAF. In all experiments, we use 5 linear segments as suggested by~\cite{agostinelli2014learning}.
\item Piecewise linear and quadratic SAAF (SAAFc1 and SAAFc2), our proposed method. In CNNs, neurons that share the same filter weights also share the same SAAF parameters. We randomly chose 22 as the number of segments, based on our proof that the model complexity can be bounded regardless of the number of polynomial segments. Break points are distributed from $-1.1$ to $1.1$.
\end{enumerate}

According to Sec~\ref{sec:desired}, AAFs on the regression neurons are especially important. Therefore we also tested the variant of applying AAFs only to regression neurons, instead of all neurons. In this case, neurons other than the regression neurons used ReLU. We distinguish AAFs only on regression neurons with prefix \textbf{R-} such as R-APLU, R-SAAFc1 \etc. Note that R-ReLU is equivalent to ReLU.

On six datasets, NNs with proposed SAAFs reduced the error of NNs with ReLU, LReLU, PReLU, APLU by \textbf{4.2-22.6\%, 6.6-20.8\%, 4.7-21.1\%, 7.4-25.0\%} respectively. Notably, on several pose estimation and age estimation datasets, we followed the training and testing set split scheme in the literature and achieved state-of-the-art results.

\subsection{Pose estimation}
\label{sec:pose}
Pose estimation is a fundamental problem in computer vision~\cite{chen2014articulated}. We focus on estimating human pose in single frame human images. Following a recent approach~\cite{belagiannis2015robust}, we address this problem by regressing a set of joint positions. There are 14 joint locations: bottom/top of head, left/right ankle, knee, hip, wrist, elbow, and shoulder. Each joint position is expressed by the $x$ and $y$ coordinates. Thus, in total there are 28 real numbers associated with each human image. Also following~\cite{belagiannis2015robust}, we used the widely used observer-centric ground truth~\cite{chen2014articulated}. We tested our method on the LSP~\cite{johnson2010clustered} and volleyball~\cite{belagiannis2014holistic} datasets containing 2000 and 1107 cropped human images respectively. We did not use the other two datasets in~\cite{chen2014articulated} because they contain only 100 to 200 training images. We used the same evaluation metric: Percentage of Correctly estimated Parts (PCP) as in~\cite{belagiannis2015robust}.

We implemented the current state-of-the-art regression NN~\cite{belagiannis2015robust} as the baseline. We used the same network architecture, Tukey's biweight loss function, and cascade of four CNNs, and just changed the CNN's activation function. Additionally, we doubled the amount of NN filters and applied extensive data augmentation on the training set. In all datasets, NNs with proposed SAAFs achieved better results than NNs with other activation functions, shown in Tab.~\ref{tab:poseoverall}.  In order to examine the effect of adding additional layers with non-adaptive activation functions compared to using SAAFs, we experimented with a ReLU NN with one additional 1K node layer. Performance increased from 62.6\% to 63.7\% whereas using R-SAAFc2 on the original smaller NN achieved 68.6\%. Using the cascade of four CNNs~\cite{belagiannis2015robust}, our proposed method (R-SAAFc2 cascade) outperformed the current state-of-the-art regression NN~\cite{belagiannis2015robust}, as shown in Tab.~\ref{tab:poseresults}. Examples of pose estimation results with our method are given in Fig.~\ref{fig:pose}.

\begin{figure}[h!]
\begin{center}
   \includegraphics[height=0.175\linewidth,clip]{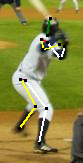}
   \includegraphics[height=0.175\linewidth,clip]{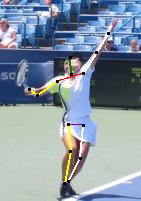}
   \includegraphics[height=0.175\linewidth,clip]{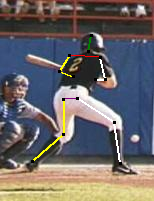}
   \includegraphics[height=0.175\linewidth,clip]{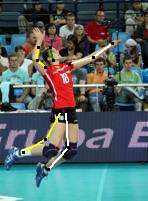}
   \includegraphics[height=0.175\linewidth,clip]{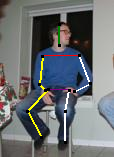}
   \includegraphics[height=0.175\linewidth,clip]{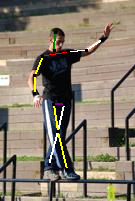}
   \includegraphics[height=0.175\linewidth,clip]{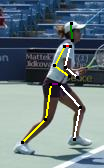}
   \includegraphics[height=0.175\linewidth,clip]{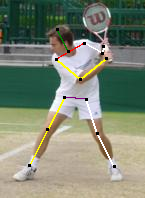}
\end{center}
   \caption{Examples of estimated poses on the LSP testing set, using our proposed method.}
\label{fig:pose}
\end{figure}

\begin{table}[h!]
\caption{Averaged PCP results of NNs with different activation functions on LSP and Volleyball pose estimation datasets. NNs with R-SAAFc2 achieved the best results.}
\centering
\begin{center}
\begin{tabular}{lrr|lrr|lrr}
\toprule
 & \multicolumn{2}{c|}{Avg. PCP} & & \multicolumn{2}{c|}{Avg. PCP} & & \multicolumn{2}{c}{Avg. PCP} \\
Methods & LSP & Volly. & Methods & LSP & Volly. & Methods & LSP & Volly. \\
\midrule
R-ReLU & 62.5 & 80.7 &  ReLU & 62.6 & 80.7 & R-SAAFc1 & 68.3 & \textbf{84.6} \\
R-LReLU & 63.1 & 81.4 &  LReLU & 63.1 & 81.4 & R-SAAFc2 & \textbf{68.6} & 84.5 \\
R-PReLU & 62.1 & 81.8 &  PReLU & 63.2 & 81.9 & SAAFc1 & 67.2 & 83.9 \\
R-APLU & 63.5 & 80.3 &  APLU & 61.9 & 80.0 & SAAFc2 & 66.6 & 82.1 \\
\bottomrule
\end{tabular}
\end{center}
\label{tab:poseoverall}
\end{table}

\begin{table}[h!]
\caption{PCP results of our proposed method vs. the current state-of-the-art regression NN on two pose estimation datasets. The cascade of 4 CNNs with R-SAAFc2 outperformed the state-of-the-art regression NN in terms of the average PCP.}
\centering
\begin{center}
\begin{tabular}{lllllllll}
\toprule
 &  &  &  & Uppr & Lowr & Uppr & Lowr & \\
Dataset & Method & Head & Torso & Legs  & Legs  & Arms & Arms & Avg. \\
\midrule
\multirow{4}{*}{LSP} & Belagiannis~\cite{belagiannis2015robust} & 72.0 & 91.5 & 78.0 & 71.2 & 56.8 & 31.9 & 63.9 \\
& Belagiannis cascade~\cite{belagiannis2015robust} &  83.2 & 92.0 & 79.9 & 74.3 & 61.3 & 40.3 & 68.8 \\
& R-SAAFc2 & 76.6 & 90.9 & 81.3 & 74.6 & 64.5 & 38.8 & 68.6 \\
& R-SAAFc2 cascade & \textbf{83.9} & \textbf{92.8} & \textbf{82.6} & \textbf{77.8} & \textbf{65.6} & \textbf{45.8} & \textbf{72.0} \\
\midrule
& Belagiannis~\cite{belagiannis2015robust} &  90.4 & 97.1 & 86.4 & \textbf{95.8} & 74.0 & 58.3 & 81.7 \\
Volley- & Belagiannis cascade~\cite{belagiannis2015robust} &  89.0 & 95.8 & 84.2 & 94.0 & 74.2 & 58.9 & 81.0 \\
ball & R-SAAFc2 &  88.6 & \textbf{99.3} & 94.7 & 94.7 & \textbf{82.5} & \textbf{60.7} & \textbf{85.3}  \\
& R-SAAFc2 cascade &  \textbf{91.5} & \textbf{99.3} & \textbf{95.2} & 94.8 & 81.0 & 54.1 & 84.1 \\
\bottomrule
\end{tabular}
\end{center}
\label{tab:poseresults}
\end{table}

\subsection{Age estimation}
\label{sec:age}
We applied our method on two age estimation datasets: the Adience~\cite{eidinger2014age} and the ICCV 2015 ChaLearn-AgeGuess challenge dataset~\cite{escalera2015chalearn}. The goal is to predict the age of people from facial images. The Adience dataset contains 26K facial images in 8 age groups, along with a cross-validation data separation scheme. We predicted the age groups as a regression problem. The Chalearn-AgeGuess (AgeGuess) dataset contains 2.4K training and 1K validation facial images. The challenge has not released the testing set, Thus we compared the validation error with other methods.

The architecture of the network was similar to the VGG 16-layer net~\cite{Simonyan14c}. We used less filters and layers. We used $1\times 1$ convolutional filters as a Network In Network (NIN) model~\cite{lin2013network}. We cropped the faces out from images using a face detection method~\cite{mathias2014face}. Then facial images were resized to 120 $\times$ 90 followed by the extensive data augmentation method useded in Sec.~\ref{sec:pose}. For experiments on the AgeGuess dataset, we first pretrained our CNNs on the Adience dataset. The results are shown in Tab.~\ref{tab:ageestimation}. We achieved state-of-the-art result on the Adience dataset. We achieved the best single CNN results on the AgeGuess dataset. Better results in the challenge~\cite{escalera2015chalearn} used extensive prediction fusions (at least 4 CNNs) and very large external datasets (at least 100X larger than the AgeGuess dataset).

\begin{table}[!htb]
    \caption{Results on age estimation datasets. For the Adience dataset, the evaluation metrics are: exact age-group match accuracy (Exact) and within 1 age-group off accuracy (1-off). We achieved state-of-the-art result on this dataset. For the Chalearn-AgeGuess (AgeGuess) dataset, we used the standard error metric used in the challenge~\cite{escalera2015chalearn}. We achieved the best single-CNN result. Better results in the challenge~\cite{escalera2015chalearn} used at least 4 CNNs and external datasets that are at least 100X larger than the AgeGuess dataset.}
\centering
\begin{tabular}{lccc|lccc}
\toprule
& \multicolumn{2}{c}{Adience} & AgeGuess & \qquad & \multicolumn{2}{c}{Adience} & AgeGuess \\
Method & Exact & 1-off & Error & Method & Exact & 1-off & Error \\
\midrule
Levi \etal~\cite{levi2015age} & 50.7 & 84.7 & - & Lab219 ~\cite{escalera2015chalearn} & - & - & 0.477 \\
R-ReLU & 49.8 & 85.1 & 0.417 & ReLU & 49.5 & 84.8 & 0.404 \\
R-LReLU & 50.2 & 84.9 & 0.423 & LReLU & 50.2 & 84.9 & 0.423 \\
R-PReLU & 51.2 & 84.7 & 0.419 & PReLU & 51.0 & 84.8 & 0.415 \\
R-APLU & 49.8 & 84.0 & 0.425 & APLU & 49.6 & 83.5 & 0.432 \\
\midrule
R-SAAFc1 & 52.9 & \textbf{88.0} & 0.404 & SAAFc1 & 53.3 & 87.7 & 0.400 \\
R-SAAFc2 & \textbf{53.5} & 87.9 & \textbf{0.387} & SAAFc2 & 52.6 & 87.7 & \textbf{0.387} \\
\bottomrule
\end{tabular}
\label{tab:ageestimation}
\end{table}

\begin{table}[h!]
\caption{RMSE and Pearson correlation results of facial attractiveness prediction using CNNs. The inter-observer agreement between three individuals is: 4.355 RMSE and 0.702 Corr. The proposed SAAFc2 achieved the best and close to human inter-observer agreement results. Matrix factorization plus visual regression achieved a correlation score of 0.478
~\cite{rothe2015some} on a similar dataset of 2000 female facial images downloaded from \texttt{hotornot.com}.}
\centering
\begin{center}
\begin{tabular}{lrr|lrr|lrr}
\toprule
Method  & RMSE & Corr. & Method & RMSE & Corr. & Method & RMSE & Corr. \\
\midrule
 R-ReLU & 4.267 & 0.656 &  ReLU & 4.193 &  0.651 & R-SAAFc1 & 3.860 & 0.739 \\
 R-LReLU & 4.132 & 0.666  & LReLU & 4.132 &  0.666 & R-SAAFc2 & 3.982 & 0.723 \\
 R-PReLU & 4.189 & 0.693  & PReLU & 4.153 & 0.706 & SAAFc1 & 3.918 & 0.734 \\
 R-APLU & 4.404 & 0.656  & APLU & 4.523 & 0.695 & SAAFc2 & \textbf{3.801} & \textbf{0.745} \\
\bottomrule
\end{tabular}
\end{center}
\label{tab:facialResults}
\end{table}

\begin{figure*}[h!]
\begin{center}
   \includegraphics[width=0.98\linewidth,trim={10 10 10 5},clip]{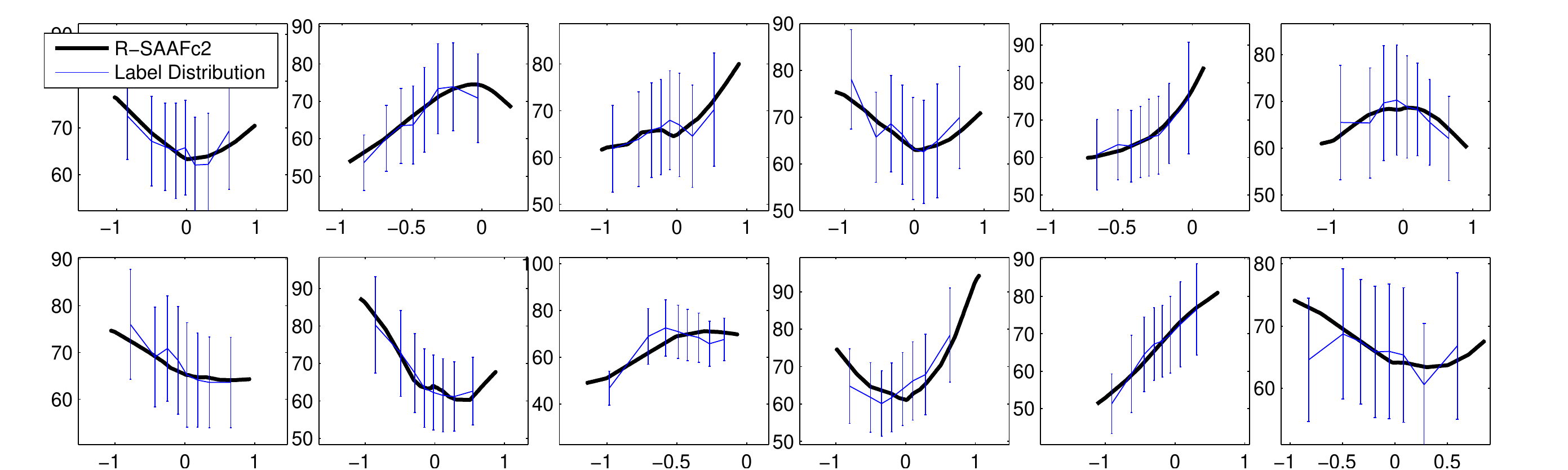}
\end{center}
   \caption{Examples of R-SAAFc2 learned on the LSP pose estimation dataset. Axes are scaled by a constant. We fed training instances into the NN to compute the inputs of regression neurons (x-axis) and the outputs of the neurons' SAAFs vs. ground truth (y-axis). We see that SAAFs of various shapes correlate with the ground truth. This graph supports our argument that the activation functions of a regression NN correlate with the ground truth, explained in Sec.~\ref{sec:desired}.}
\label{fig:smthact_pose}
\end{figure*}

\subsection{Facial attractiveness prediction}
We focus on learning personal preferences of facial attractiveness on web-quality facial images. We constructed a dataset by applying the Viola-Jones' face detection algorithm~\cite{viola_jones_2001a} on female images downloaded from \texttt{hotornot.com}. The dataset contains 2K RGB facial images of 120$\times$90 pixels. Three individuals independently rated the images with facial attractiveness scores ranging from 0 to 20. We randomly selected one of the three raters to provide the training and testing labels. Labels from the other two raters were only used to compute inter-observer agreement. We used the same CNN architecture in Sec.~\ref{sec:age}. Results are shown in Tab.~\ref{tab:facialResults}. The CNN with our SAAFc2 achieved the best three fold random-split-validation results. Our results are close to inter-observer agreement.

\begin{table}[h!]
\caption{RMSE and Pearson correlation results of learning the circularity of nuclei. Filters in all CNNs were initialized with a convolutional auto-encoder~\cite{masci2011stacked}. Directly computing circularity from automatically segmented nuclei achieved an RMSE of 0.609 and a correlation score of 0.493. It would give zero RMSE if the nuclei segmentation was perfect. Our R-SAAFc1 performed the best.}
\centering
\begin{center}
\begin{tabular}{lrr|lrr|lrr}
\toprule
Method  & RMSE & Corr. & Method & RMSE & Corr. & Method & RMSE & Corr. \\
\midrule
R-ReLU & 0.616 & 0.620 &  ReLU & 0.623 & 0.611 &  R-SAAFc1 & \textbf{0.483} & \textbf{0.649} \\
R-LReLU & 0.610 & 0.621 &  LReLU & 0.610 & 0.621  & R-SAAFc2 & 0.492 & 0.641 \\
R-PReLU & 0.508 & 0.631 &  PReLU & 0.517 & 0.622  & SAAFc1 & 0.513 & 0.634 \\
R-APLU & 0.618 & 0.626 &  APLU & 0.587 & 0.629 &  SAAFc2 & 0.498 & 0.642 \\
\bottomrule
\end{tabular}
\end{center}
\label{tab:nuclearResults}
\end{table}

\subsection{Learning the circularity of nuclei in pathology images}
Hematoxylin and eosin stained pathology images provide rich information to diagnose, classify and study cancer. One diagnostic criterion is the shape of nuclei~\cite{braak2003staging}. We use CNNs to learn the circularity of nuclei in pathology images of glioma which is the most common brain cancer. Existing methods analyze the shape of nuclei from automatically segmented nuclei~\cite{al2010improved}. We used the training set from the MICCAI 2015 nucleus segmentation challenge~\cite{miccainucleussegmentation} which contains 1K images of nuclei. We derived the ground truth circularity measurements from the ground truth nuclear segmentation masks and used CNNs to learn the circularity from nuclear images. Experiments show that our CNN-predicted circularity is more accurate than the circularity directly computed from automatically segmented nuclei. In the results, shown in Tab.~\ref{tab:nuclearResults}, the CNN with our R-SAAFc1 achieved the best three fold random-split-validation result. Note that, data-driven regression achieved better results, compared to computing circularity of nuclei directly from automatic nuclei segmentation results.

\section{Conclusions}
\label{sec:conclusion}
We have demonstrated theoretically and experimentally that using Adaptive Activation Functions (AAF) on the regression (second-to-last) layer can improve the performance of a regression NN. We proposed a novel AAF named Smooth Adaptive Activation Function (SAAF) which has multiple advantages. First, an SAAF can approximate any function to a desired degree of error. Second, using parameter regularization, an NN with SAAFs represents a Lipschitz continuous function which leads to bounded model complexity in terms of the fat-shattering dimension. Based on these two advantages, an NN with SAAFs can achieve lower model bias and complexity than NNs with other AAFs. We tested different setup of SAAFs in various NN architectures on several real world datasets. The improvements are consistent across all tested datasets compared to adaptive and non-adaptive activation functions. We achieved state-of-the-art results on multiple datasets. In the future, we will test SAAFs in classification NNs.

{\small
\bibliographystyle{ieee}
\bibliography{reference}
}

\end{document}